\documentclass{article}

\usepackage{times}
\usepackage{graphicx}
\usepackage{natbib}

\usepackage{algorithm}
\usepackage{algorithmic}

\usepackage{amssymb}
\usepackage{multirow}
\usepackage{amsmath}
\usepackage{amsthm}
\usepackage{color}
\usepackage{mathrsfs}
\usepackage{graphicx}
\usepackage{caption, subcaption}
\usepackage{xspace}
\usepackage{float}

\usepackage{hyperref}

\usepackage[accepted]{icml2016}

\def\to{{\,\rightarrow\,}}

\mathchardef\mhyphen="2D
\newcommand{\smallmat}[1]{\left[\begin{smallmatrix}#1\end{smallmatrix}\right]}

\newcommand{\tr}[1]{{\mathrm{tr}}\!\left( #1 \right)}

\newcommand{\vertiii}[1]{{\left\vert\kern-0.25ex\left\vert\kern-0.25ex\left\vert #1
    \right\vert\kern-0.25ex\right\vert\kern-0.25ex\right\vert}}

\newcommand{\vect}[1]{{\boldsymbol{#1}}}

\def\balpha{\vect{\alpha}}

\def\bm{{\mathbf{m}}}

\def\bp{{\mathbf{p}}}
\def\bq{{\mathbf{q}}}
\def\br{{\mathbf{r}}}
\def\bs{{\mathbf{s}}}
\def\bt{{\mathbf{t}}}
\def\bu{{\mathbf{u}}}
\def\bv{{\mathbf{v}}}

\def\bx{{\mathbf{x}}}
\def\by{{\mathbf{y}}}
\def\bz{{\mathbf{z}}}
\def\bA{{\mathbf{A}}}
\def\bB{{\mathbf{B}}}

\def\bG{{\mathbf{G}}}

\def\bI{{\mathbf{I}}}

\def\bR{{\mathbf{R}}}

\def\bT{{\mathbf{T}}}

\def\bX{{\mathbf{X}}}
\def\bY{{\mathbf{Y}}}
\def\bZ{{\mathbf{Z}}}

\def\bbR{{\mathbb{R}}}

\def\cA{\mathcal{A}}

\def\cH{\mathcal{H}}
\def\cI{\mathcal{I}}
\def\cJ{\mathcal{J}}

\def\cS{\mathcal{S}}

\def\sfI{\mathsf{I}}

\def\sfP{\mathsf{P}}

\newtheorem{theorem}{Theorem}
\newtheorem{definition}[theorem]{Definition}
\newtheorem{corollary}[theorem]{Corollary}
\newtheorem{lemma}[theorem]{Lemma}
\newtheorem{proposition}[theorem]{Proposition}
\renewcommand{\text}[1]{{\textnormal{#1}}}

\DeclareMathOperator*{\argmin}{arg\,min}
\DeclareMathOperator*{\argmax}{arg\,max}
\DeclareMathOperator{\conv}{conv}
\DeclareMathOperator{\lin}{lin}
\DeclareMathOperator{\nullSp}{null}
\newcommand{\lmo}{\textsc{LMO}\xspace}

\icmltitlerunning{Pursuits in Structured Non-Convex Matrix Factorizations}

\begin{document}

\twocolumn[

\icmltitle{Pursuits in Structured Non-Convex Matrix Factorizations}

\icmlauthor{Rajiv Khanna}{rajivak@utexas.edu}
\icmladdress{UT Austin}
\icmlauthor{Michael Tschannen}{michaelt@nari.ee.ethz.ch}
 \icmladdress{ETH Zurich}
\icmlauthor{Martin Jaggi}{jaggi@inf.ethz.ch}
 \icmladdress{ETH Zurich}
  
\vskip 0.3in
]

\begin{abstract}
Efficiently representing real world data in a succinct and parsimonious manner is of central importance in many fields.  We present a generalized greedy pursuit framework, allowing us to efficiently solve structured matrix factorization problems, where the factors are allowed to be from arbitrary sets of structured vectors. Such structure may include sparsity, non-negativeness, order, or a combination thereof. The algorithm approximates a given matrix by a linear combination of few rank-1 matrices, each factorized into an outer product of two vector atoms of the desired structure. For the non-convex subproblems of obtaining good rank-1 structured matrix atoms, we employ and analyze a general atomic power method.
In addition to the above applications, we prove linear convergence for generalized pursuit variants in Hilbert spaces --- for the task of approximation over the linear span of arbitrary dictionaries --- which generalizes OMP and is useful beyond matrix problems.
Our experiments on real datasets confirm both the efficiency and also the broad applicability of our framework in practice.
\end{abstract}

\setlength{\belowdisplayskip}{5pt} \setlength{\belowdisplayshortskip}{4pt}
\setlength{\abovedisplayskip}{5pt} \setlength{\abovedisplayshortskip}{4pt}

\section{Introduction}\vspace{-1mm}
Approximating a matrix using a structured low-rank matrix factorization is a cornerstone problem in a huge variety of data-driven applications. 
This problem can be seen as projecting a given matrix onto a linear combination of few rank-1 matrices, each of which being an outer product of two vectors, each from a structured set of vectors. Examples of such structure in the vectors can be sparsity, group sparsity, non-negativeness etc. The structure is generally encoded as a constraint on each of the two factors of the factorization problem. 
Even without imposing structure, the rank-constrained problem is already NP-hard to solve in general\footnote{
For example, least-squares matrix completion is NP hard even for a rank-1 factorization, as shown by \cite{Gillis:2011ia}.
}. Instead of a rank constraint, convex relaxations are therefore typically applied. However, this involves giving up explicit control over the resulting rank. 
Nevertheless,  there has been a strong body of research studying recovery and performance under several convex relaxations of rank-constrained problems \cite{Candes:2009kj,Candes:2010jb,Toh:2009we,Pong:2010cg}.

In this paper, we take a different approach. We keep explicit control over the rank of the factorization, as well as the precise structure of the used (vector) factors, which we call atoms. Our approach is a greedy method adding one rank-1 atom per outer iteration, as in matrix variants of matching pursuit \cite{wang2014matrixcompletion} as well as Frank-Wolfe algorithms on factorizations \cite{Hazan:2008kz,Jaggi:2010tz,Dudik:2012ts,jaggi13FW,Bach:2013wk}.
By keeping the explicit low-rank factorization into the vector atoms at all times, we can study general algorithms and correction variants applying directly to the original non-convex problems.

Iteratively adding a rank-1 atom for structured matrix factorization falls into the purview of pursuit algorithms, which we here present in general form. 
Each update is obtained from a linear minimization oracle (LMO), which outputs the best rank-1 atom with respect to a linearized version of the objective. 
Each iteration hence increases the rank by 1, while improving the approximation quality. We will systematically study this tradeoff between rank and approximation quality, by providing convergence rates of the iterates to the best possible structured approximation of the given matrix.
 
To study the convergence rates for matrix pursuit over structured sets, we adapt the convergence results for pursuit algorithms from the compressed sensing literature. While most of the existing work focuses on assessing the quality of $k$-greedy selection of atoms vis-a-vis a best possible $k$-atom selection, there is some work that discusses linear convergence of greedy pursuit approaches in an optimization sense~\citep{Davis:1997fh,Mallat:1993gu,Gribonval:2006ch,Blumensath:2008fs,jones87,dupe2015}.
However, they are not directly applicable to matrix pursuit for structured factorization, because they typically require the observed vector to lie within the span of the given dictionary (i.e., the structured sets) and make strong structural assumptions about the dictionary (e.g., incoherence, see Section \ref{sec:coherence}). We show that even in the more general case when the observation is not in the span of dictionary, greedy pursuit in a Hilbert space converges linearly to the best possible approximation.
In the language of functional analysis, the line of research of~\citet{DeVore:2014ti,Temlyakov:2014eb} is closest to this approach. 
Note that such results have a wider scope than just the applications to matrix pursuit.
 
The general nature of our convergence result allows its application to any set of atoms which induce an inner product -- which in turn induces a distance function, that can be minimized by  greedy pursuit. It is easy to show that the low rank structured matrix completion problem can be cast in this framework as well, and this yields an algorithm that works for any atomic vector set. 
For the specific case of atomic vector sets being unit 2-norm balls without any further structure, this setup was used by~\citet{wang2014matrixcompletion}, who showed linear convergence for matrix pursuit on 2-norm balls as vector atomic sets. This is a special case of our framework, because we show linear convergence with \emph{any} compact vector atomic sets.  We also present empirical results on real world datasets that show that this generalization is useful in practice. 

The linear convergence of the matching pursuit in Hilbert spaces specifies the decay in reconstruction error in terms of number of calls made to the LMO. For the matrix pursuit algorithm, the linear problem being solved by the LMO itself may be a NP-hard, though efficient solutions are available in some cases~\citep{Bach:2008wn,Recht:2010tf}. We also analyze the pursuit using only an approximate version of the LMO, 
and we show that the linear convergence rate is still maintained, but the decay is less sharp depending on the approximation quality of the LMO. 

\textbf{Related work:} There exists a vast literature on structured matrix factorizations. For our cases, the most relevant are the lines of research with iterative rank-1 greedy approximations such as the Frank-Wolfe algorithm~\cite{Hazan:2008kz,Jaggi:2010tz,Dudik:2012ts,jaggi13FW,Bach:2013wk}. In the tensor case, a very similar approach has recently been investigated by \citet{Yang:2015wy}, but not for the case of structured factorizations like we do here. Their linear convergence result is also a special case of our more general convergence rate result in Hilbert spaces.  Similarly, for specific atomic sets, there is a large body of literature, see, e.g., \citep{yuan2013,journee2010,papailiopoulos_sparsepca} for Sparse PCA, \citep{sigg2008,asteris2014} for sparse non-negative PCA, and references therein.

There is a significant amount of research on pursuit algorithms, even more so on one of its more commonly used flavors known as orthogonal matching pursuit. \citet{Davis:1997fh} prove geometric convergence of matching pursuit and its orthogonal counterpart for finite dictionaries, while \citet{Mallat:1993gu,Gribonval:2006ch} give convergence results for (quasi-)incoherent dictionaries in Hilbert spaces of finite or infinite dimension. However, all of these assume the observed vector to lie in the dictionary span, so that the goal is to exactly re-construct it using as few atoms as possible rather than to approximate it using as few atoms as possible. 
For infinite-dimensional pursuit,~\citet{jones87} showed convergence without providing rates. 

The matrix completion problem has gained significant interest recently, motivated by powerful applications in recommender systems (e.g. Netflix prize, \citet{Koren:2009jg}), signal processing (robust PCA, \citet{Candes:2011bd}), and most recently word-embeddings in NLP \cite{Levy:2014vd}.  
\citet{Candes:2009kj,recht2009} and several subsequent works study the completion problem by convex optimization and provide bounds on exact matrix completion for random matrices.~\citet{jain2013} provide guarantees for low rank matrix completion by alternating minimization under incoherence. These works and several followups cast the matrix completion as minimization of the rank of the matrix (or a convex surrogate) under the constraint that the observed entries are reproduced exactly or approximately. A matrix pursuit view of the problem was taken by~\citet{wang2014matrixcompletion} by adding rank-1 updates iteratively to decrease the reproduction error on the observed entries.

\paragraph{Contributions.} Our key contributions are as follows: \vspace{-1mm}
\begin{itemize}

 \item We devise and analyze a general algorithmic framework for structured matrix factorizations, where the factors are allowed to be from an arbitrary set of structured vectors. Our method only assumes a constrained linear minimization oracle (LMO), and can be seen as a special case of a more general class of pursuit algorithms, which we analyze in Section~\ref{sec:matchingpursuit}.

 \item We prove a linear convergence guarantee for generalized pursuit in Hilbert spaces for approximation over the linear span of arbitrary dictionaries, which generalizes OMP and is useful beyond matrix problems. 
 \item For the non-convex rank-one factorization subproblems per iteration, we propose a general atomic power method, allowing to efficiently approximate the LMO for arbitrary structured sets of vector atoms.
\item We improve efficiency of the resulting methods in terms of the required rank (number of atoms) by allowing corrective variants of the algorithm.
\item Finally, we provide strong experimental results on real world datasets, confirming the efficiency and broad applicability of our framework in practice.
\end{itemize}

\paragraph{Notation.} We represent vectors as small letter bolds, e.g., ~$\bu$.
Matrices are represented by capital bolds, e.g., $\bX, \bT$.
Vector/matrix transposes are represented by superscript $\top$. Identity matrix is represented as $\bI$. Sets are represented by calligraphic letters, e.g., $\cS$. 
For a matrix $\bA \in \bbR^{m \times n}$ and a set of index pairs~$\Omega$,~$\bA_\Omega$ stands for the matrix that is equal to $\bA$ at the entries indexed by $\Omega$, and $0$ elsewhere.
Let $[d]$ be the set $\{1,2,\ldots,d\}$. 
Let $\conv(\cS)$ be the convex hull of the set~$\cS$, and let $\lin(\cS)$ denote the linear span of the elements in~$\cS$. $\bu \otimes \bv$ represents the rank-1 matrix given by the outer product $\bu\bv^\top$ of two vectors $\bu,\bv$. Analogously we write $\cA_1 \otimes \cA_2$ for the set of outer products between any pair of elements from two sets $\cA_1$ and $\cA_2$ respectively.

\section{Generalized Pursuit in Hilbert Spaces}
\label{sec:matchingpursuit}
In this section, we develop a generalized pursuit algorithm for Hilbert spaces. Let $\cH$ be a Hilbert space with associated inner product $\langle \bx, \by\rangle_\cH, \,\forall \, \bx,\by \in \cH$. The inner product induces the norm $\| \bx \|_\cH^2 := \langle \bx, \bx \rangle_\cH,$ $\forall \, \bx \in \cH$, as well as the distance function $d_{\cH} (\bx, \by) := \| \bx - \by\|_\cH,$ $\forall \, \bx,\by \in \cH$. 

Let $\cS \subset \cH$ be a bounded set and let $f_\cH \colon \cH \to \bbR$. We would like to solve the optimization problem

\begin{equation}
  \min_{\bx \in \lin(\cS)} f_{\cH}(\bx).
 \label{eqn:pursuitgoal}
\end{equation}
We write $\bx^\star$ for a minimizer of \eqref{eqn:pursuitgoal}.
For any $\by \in \cH$, and bounded set $\cS \subset \cH$, a linear minimization oracle (\lmo) is defined as
\begin{equation*}
\lmo_\cS(\by) := \argmin_{\bx \in \cS} \, \langle \bx, \by\rangle_\cH .
\end{equation*}

\begin{algorithm}[h]
\caption{ Generalized Pursuit (GP) in Hilbert Spaces}
\begin{algorithmic}[1]
\REQUIRE $\bx_0 \in \cS$, $R$, $f_\cH$ 
\FOR{$r$ = 1..$R$}
\STATE $\bz_r := (Approx-)\lmo_{\cS}\big( \nabla f_\cH(\bx_{r-1}) \big)$ 
\STATE  $\balpha := \min_{\balpha \in \bbR^r} f_\cH(\sum_{i \leq r} \alpha_i \bz_i)$
\STATE \emph{Optional:} Correction of some/all $\bz_i$.
\STATE Update iterate $\bx_{r} := \sum_{i \leq r}\alpha_i \bz_i$
\ENDFOR
\end{algorithmic}
\label{algo:GP}
\end{algorithm}

Our generalized pursuit algorithm is presented as Algorithm~\ref{algo:GP}. 
In each iteration $r$, the linearized objective at the previous iterate 
is minimized over the set $\cS$ (which is purpose of the \lmo), in order to obtain the next atom $\bz_r$ to be added.

After that, we do allow for corrections of the weights of all previous atoms, as shown in line~3. In other words, we obtain the next iterate by minimizing $f_\cH$ over the linear span of the current set of atoms. This idea is similar to the fully-corrective Frank-Wolfe algorithm and its away step variants~\cite{LacosteJulien:2015wj}, as well as orthogonal matching pursuit (OMP) \cite{chen1989orthogonal}, as we will discuss below.
E.g. for distance objective functions $d_{\cH}^2$ as of our interest here, this correction step is particularly easy and efficient, and boils down to simply solving a linear system of size $r\times r$.

A second type of (optional) additional correction is shown in line~4, and allows to change some of the actual atoms of the expansion, see e.g.  \citet{Laue:2012wn}. 
Both types of corrections have the same motivation, namely to obtain better objective cost while using the same (small) number of atoms~$r$.

The total number of iterations $R$ of Algorithm~\ref{algo:GP} controls the the tradeoff between approximation quality, i.e., how close $f_\cH(\bx_R)$ is to the optimum $f_\cH(\bx^\star)$, and the ``structuredness'' of $\bx_R$ due to the fact that we only use $R$ atoms from $\cS$ and through the structure of the atoms themselves (e.g., sparsity).
If $\cH$ is an infinite dimensional Hilbert space, then $f$ is assumed to be \emph{Fr{\'e}chet differentiable}.

Note that the main difference between generalized pursuit as in Algorithm~\ref{algo:GP} and Frank-Wolfe type algorithms (both relying on the same \lmo) is that pursuit maintains its respective current iterates as a \emph{linear} combination of the selected atoms, while in FW restricts to \emph{convex} combinations of the atoms.

We next particularize Algorithm \ref{algo:GP} to the least squares objective function $f_\cH(\bx) := \frac{1}{2}d^2_\cH (\bx, \by)$ for fixed $\by \in \cH$. As will be seen later, this variant of Algorithm \ref{algo:GP} has many practical applications. The minimization in line 3 of Algorithm \ref{algo:GP} hence amounts to orthogonal projection of $\by$ to the linear span of $\{\bz_i; i \leq r\}$, and $-\nabla f_\cH(\bx_{r})$ writes as $\br_{r+1} := \by - \bx_r$, henceforth referred to as the residual at iteration $r$. 
We thus recover a variant of the OMP algorithm \cite{chen1989orthogonal}. By minimizing only w.r.t. the new weight $\alpha_r$ in line 3 of Algorithm \ref{algo:GP} we further recover a variant of the matching pursuit (MP) algorithm \cite{Mallat:1993gu} \footnote{In OMP and MP, the \lmo consist in solving $\argmax_{\bx \in \cS} \, |\langle \bx, \br_{r-1}\rangle_\cH|$ at iteration $r$.}.

Our first main result characterizes the convergence of Algorithm \ref{algo:GP} for $f_\cH(\bx) := \frac{1}{2}d^2_\cH (\bx, \by)$.

\begin{theorem}
\label{thm:convExactHilbert}
Let $f_\cH(\bx) := \frac{1}{2}d^2_\cH (\bx, \by)$ for $\by \in \cH$. For every $\bx_0 \in \cS$, Algorithm~\ref{algo:GP} converges linearly to $\bx^\star$.
\end{theorem}
\vspace{-2mm}

\begin{table*}[t]
\centering
\caption{Some example applications for our matrix pursuit framework (Algorithm~\ref{algo:GMP}). The table characterizes the applications by the set of atoms used as to enforce matrix structure \emph{(rows)}, and two prominent optimization objectives \emph{(columns)}, being low-rank matrix approximation and low-rank matrix completion (MC). All cases apply for both symmetric as well as general rectangular matrices.
\vspace{1mm}
}
\label{tab:Applications}

\subcaption{Symmetric Structured Matrix Factorizations, $\bu_i \in \cA_1\ \ \forall i$}
\label{tab:subtabSymm}
\begin{tabular}{|l|c|c|}
\hline
\textbf{Atoms} $\cA_1$         & $f\!= \| \bY - \sum_i \alpha_i  \bu_i \otimes \bu_i\|_F^2$ & $f\!= \| \bY - \sum_i \alpha_i  \bu_i \otimes \bu_i \|_\Omega^2$   \\\hline\hline
$\{ \bu: \|\bu\|_2\!=\!1 \}$   & PCA   & MC     \\\hline
$\{ \bu: \|\bu\|_2\!=\!1$, $\|\bu\|_0=k \}$   & sparse PCA & structured MC \\\hline

\end{tabular}

\bigskip
\subcaption{Non-Symmetric Structured Matrix Factorizations,  $\bu_i \in \cA_1,\ \bv_i \in \cA_2\ \ \forall i$}
\label{tab:subtabnonsymm}
\setlength\tabcolsep{2pt}
\begin{tabular}{|l|l|c|c|}
\hline
\textbf{Atoms} $\cA_1$          & \textbf{Atoms} $\cA_2$         & $f\!= \| \bY \!-\! \sum_i \alpha_i  \bu_i \otimes \bv_i\|_F^2$ & $f\!= \| \bY \!-\! \sum_i \alpha_i  \bu_i \otimes \bv_i\|_\Omega^2$    \\\hline\hline
$\{ \bu: \|\bu\|_2\!=\!1 \} $  & $\{ \bv: \|\bv\|_2\!=\!1 \}$ & SVD & MC     \\\hline
$\{ \bu: \|\bu\|_2\!=\!1$, $\|\bu\|_0=k\}$ & $\{ \bv: \|\bv\|_2\!=\!1$, $\|\bv\|_0=q\}$& sparse SVD & structured MC \\\hline
$\{ \bu: \|\bu\|_2\!=\!1, \bu \geq \mathbf{0} \}$ & $\{ \bv: \|\bv\|_2\!=\!1, \bv \geq \mathbf{0} \}$& NMF & structured MC \\\hline
$\{ \bu: \|\bu\|_2\!=\!1, \bu \geq \mathbf{0}, \|\bu\|_0\!=\!k \}$ & $\{ \bv: \|\bv\|_2\!=\!1, \bv \geq \mathbf{0}, \|\bv\|_0\!=\!q \}$& sparse NMF & structured MC \\\hline
\end{tabular}

\end{table*}

\paragraph{Discussion.}
The linear convergence guarantees $f_\cH(\bx_r)$ to be within $\epsilon$ of the optimum value $f(\bx^\star)$ after $O(\log1/\epsilon)$ iterations.  Moreover, since the atoms selected across iterations are linearly independent, the maximum number of iterations in the finite-dimensional case is equal to the dimension of $\cH$. 
\\
Algorithm~\ref{algo:GP} allows for an \emph{inexact \lmo}. In particular for our main interest of matrix factorization problems here, an exact \lmo is often very costly, while approximate versions can be much more efficient. We refer to the supplementary material for the the proof that the linear convergence rate still holds for approximate \lmo, in terms of multiplicative error.

\subsection{Relationship with coherence-based rates} \label{sec:coherence}
Our presented convergence analysis (which holds for infinite sets of atoms) can be seen as a generalization of existing coherence-based rates which were obtained for finite dictionaries $\cS = \{ \bs_i; i \in [n]\} \subset \cH$. Throughout this section, we assume that the atoms are normalized according to $\| \bs \|_\cH = 1, \forall \bs \in \cS$, and that $\cS$ is symmetric, i.e., $\bs \in \cS$ implies $-\bs \in \cS$. 

In the general case, we assume $\by$ to lie in $\cH$ (not necessarily in $\cS$). In the sequel, we will restrict $\by$ to lie in $\lin(\cS)$ in order to recover a known coherence-based convergence result for OMP for finite dictionaries. Our results are based on the cumulative coherence function.

\begin{definition}[Cumulative coherence function~\cite{Tropp:2004gc}] Let $\cI \subset [n]$ be an index set. For an integer $m$, cumulative coherence function is defined as:
\begin{equation*}
\mu(m) := \max_{|\cI| = m} \max_{k \in [n]\backslash \cI} \sum_{i \in \cI} |  \langle \bs_k, \bs_i \rangle_\cH | \vspace{-2mm}
\end{equation*}
\end{definition}

Note that $\mu(m) \leq m \mu(1)$, and $\mu(1)$ can be written as $\max_{j \neq k} |\langle \bs_j, \bs_k \rangle_\cH |$. 
Using the cumulative coherence function to restrict the structure of $\cS$, we obtain the following result.

\begin{theorem} [Coherence-based convergence rate]
\label{thm:coherence}
Let $\br^\star := \by - \bx^\star$. 

If $\mu(n-1) < 1$, where $n= |\cS|$, then the residuals follow linear convergence as
 \begin{equation*}
 \| \br_{r+1} - \br^\star \|_\cH^2 \leq  \left( 1 - \frac{1 - \mu(n-1)}{n}\right) \| \br_{r} - \br^\star \|_\cH^2 \,.
 \vspace{-2mm}
 \end{equation*}
\end{theorem}

Writing out the condition $\mu(n-1) < 1$ yields $\max_{k \in [n]} \sum_{i \in [n] \setminus k} |  \langle \bs_k, \bs_i \rangle_\cH | < 1$, which implies that the atoms in $\cS$ need to be close to orthogonal when the number of atoms is on the order of the ambient space dimension. Thus, Theorem \ref{thm:coherence} gives us an explicit convergence rate at the cost of imposing strong structural conditions on $\cS$.

Considering the special case of our Theorem~\ref{thm:convExactHilbert} for achievable $\by \in \lin(\cS)$, finite dictionary, and $f_\cH(\bx) := \frac{1}{2}d^2_\cH (\bx, \by)$, we recover the following known result for the \emph{linear} convergence of OMP under an analogous coherence bound on~$\cS$:

\begin{corollary}[{see also \citet[Thm. 2b]{Gribonval:2006ch}}]
\label{cor:coherenceomp}

If $\by \in \lin(\cS)$, then $ \forall 1 \leq r \leq m \text{ s.t. }\mu(m-1) < 1, m \leq n$, 
\begin{equation*}
 \| \br_{r+1}  \|_\cH^2 \leq  \left( 1 - \frac{1 - \mu(m-1)}{m}\right) \| \br_{r}  \|_\cH^2.
\end{equation*}
\end{corollary}

Proofs of Theorems \ref{thm:coherence} and Corollary \ref{cor:coherenceomp} are provided in the supplementary material.

\section{Matrix Pursuit}
\label{sec:matrixPursuit}

Motivated from the pursuit algorithms from the previous section, we here present a generalized pursuit framework for structured matrix factorizations.

In order to encode interesting structure for matrix factorizations, we study the following class of matrix atoms, which are simply constructed by arbitrary two sets of vector atoms $\cA_1 \subseteq \mathbb{R}^n$ and $\cA_2 \subseteq \mathbb{R}^m$.
We will study pursuit algorithms on the set of rank-1 matrices $\cA_1 \otimes \cA_2$, each element being an outer product of two vectors.

Specializing the general optimization problem \eqref{eqn:pursuitgoal} to sets $\lin(\cA_1 \otimes \cA_2)$, we obtain the following structured matrix factorization notion:
Given an objective function $f\colon \bbR^{n \times m} \to \bbR$, we want to find a matrix $\bX$ optimizing
\begin{equation}
\label{eq:matcostfuncfactorization}
\min_{\bX \in \lin(\cA_1 \otimes \cA_2)} f(\bX).
\end{equation}
When restricting \eqref{eq:matcostfuncfactorization} to candidate solutions of rank at most~$R$, we obtain the following equivalent and more interpretable factorized reformulation:\vspace{-1mm}
\begin{equation}
\label{eq:costfuncfactorization}
\min_{\substack{\bu_i \in \cA_1\ \forall i\in[R],\\ \bv_i \in \cA_2 \ \forall i\in[R],\\ \alpha \in\bbR^R}} 
f\bigg(\sum^R_{i=1} \alpha_i \bu_i \otimes \bv_i \bigg).\vspace{-1mm}
\end{equation}

\paragraph{Symmetric factorizations.}
The above problem structure is also interesting in the special case of symmetric matrices, when restricting to just one set of vector atoms~$\cA_1$ (and $\bu = \bv$), which results in symmetric matrix factorizations build from atoms of the form~$\bu\bu^\top$. 

\vspace{-1mm}
\paragraph{Applications.} 
We present some prominent applications of structured matrix factorizations within our pursuit framework in Table~\ref{tab:Applications}.

The vector atom sets $\cA_1$ and $\cA_2$ encode the desired matrix factorization structure. For example, in the special case $f(\sum_i \alpha_i \bu_i \otimes \bv_i)= \| \bY - \sum_i \alpha_i \bu_i \otimes \bv_i \|^2_F$ for a given matrix~$\bY$, and atoms $\cA_1 = \cA_2 = \{ \bx: \| \bx \|_2 = 1\}$, problem~\eqref{eq:matcostfuncfactorization} becomes the standard SVD.

Typical structures of interest for matrix factorizations include sparsity of the factors in various forms, including group/graph structured sparsity (see~\citep{baraniuk2010} and references therein for more examples). Furthermore, non-negative factorizations are widely used, also ordered vectors and several other structures on the atom vectors. In our framework, it is easy to also use combinations of several different vector structures. Also, note that the sets $\cA_1$ and $\cA_2$ are by no means required to be of the same structure. For the rest of the paper, we assume that the sets $\cA_1$ and $\cA_2$ are compact.

\paragraph{Algorithm.}
The main matrix pursuit algorithm derived from Algorithm~\ref{algo:GP} applied to problems of form \eqref{eq:costfuncfactorization} is presented in Algorithm~\ref{algo:GMP}. 
\vspace{-1em}

\begin{algorithm}[H]
\caption{ Generalized Matrix Pursuit (GMP)}
\begin{algorithmic}[1]
\REQUIRE $\bX_0$, $R$
\FOR{$r$ = 1..$R$}
\STATE $\bu_r, \bv_r := (Approx-)\lmo_{\cA_1 \otimes \cA_2}\big( \nabla f(\bX_{r-1}) \big)$
\STATE  $\balpha := \min_{\balpha \in \bbR^r} f(\sum_{i \leq r} \alpha_i \bu_i \otimes \bv_i) $
\STATE \emph{Optional:} Correction of some/all $\bu_i, \bv_i$.
\STATE Update iterate $\bX_r := \sum_{i \leq r}\alpha_i \bu_i \bv_i^\top$
\ENDFOR
\end{algorithmic}
\label{algo:GMP}
\end{algorithm}

In practice, the atom-correction step (Step 4) is specially important for maintaining iterates of even smaller rank in practice, as also highlighted by our experiments. \emph{Local} corrections are made to the already chosen set of atoms to potentially improve the quality of rank-$r$ solution. 

\paragraph{Matrix completion.}
Variants of (structured) matrix completion are obtained for the objective function\vspace{-1mm}
\begin{equation}
f \Big(\sum_i \alpha_i \bu_i \otimes \bv_i \Big) = 
\Big\| \bY - \sum_i \alpha_i  \bu_i \otimes \bv_i \Big\|_\Omega^2,
\label{eq:fcompletion}
\end{equation}
where $\Omega$ is set of observed indices. Here the norm on the vector space is defined with respect to only the observed entries. Formally, $\|\bZ\|_\Omega^2 = \|\bZ_\Omega\|_F^2$ is induced by the inner product $\langle  \bA, \bB \rangle_\Omega :=  \tr{ \bA_\Omega^\top \bB_\Omega}$.

\paragraph{Convergence.}
The linear rate of convergence proved in Theorem~\ref{thm:convExactHilbert} is directly applicable to Algorithm~\ref{algo:GMP} as well. This is again subject to the availability of a linear oracle (\lmo) for the used atoms.
\\
The convergence rate presented by~\citet{wang2014matrixcompletion} for the case of matrix completion can be obtained directly as a special case of our Theorem~\ref{thm:convExactHilbert}, for $\cA_1:=\{\bu:\|\bu\|_2=1\}, \cA_2:=\{\bv:\|\bv\|_2=1\} $ and $\bR^\star := \bY - \bX^\star = \mathbf{0}$.

\paragraph{Generalized rank for structured factorizations.}
For the case of given $\by \in \lin(\cS)$, the number of iterations performed by Algorithm~\ref{algo:GMP} can be thought of as a complexity measure of generalized matrix rank, specific to our objective function $f_\cH$ and the atomic sets.
As an example, one can directly obtain the analogue of the $k$-$q$ rank of matrices for sparse SVD defined by~\citet{richard2014}.

\subsection{Atom Correction Variants}\label{sec:corrective}
Algorithms~\ref{algo:GP} and \ref{algo:GMP} guarantee linear convergence in terms of number of calls made to the \lmo oracle, each iteration increasing the rank of the iterate by one.
In many applications such as low rank PCA, low rank matrix completion etc., it is desirable to have iterates being a linear combination of only as few atoms as possible. As discussed in the previous Section \ref{sec:matchingpursuit}, we do allow for corrections to obtain a possibly much lower function cost with a given fixed rank approximation.
The more severe corrections of atoms themselves in step~4 (as opposed to just their weights) can be made by updating them one or a few atoms at a time, keeping the rest of them fixed.
For the symmetric case, the update of the $i^{\text{th}}$ atom can be written as 
\begin{equation}
\label{eq:updatesGeneralSymmetric}
\bu_i^+ :=   \argmin_{\bu \in \cA_1} f\Big(\sum_{j \neq i} \alpha_j \bu_j \otimes \bu_j + \alpha_i \bu \otimes \bu \Big).\vspace{-2mm}
\end{equation} 
The update for non-symmetric case is analogous. The complexity of atom corrections depends very strongly on the structure of the used atomic sets $\cA_1$ and $\cA_2$. 
One general way is to call the \lmo again, but assuming the current iterate is $\sum_{j \neq i} \alpha_j \bu_j \otimes \bu_j $. For the non-symmetric case, techniques such as alternating minimization between~$\bu_i$ and~$\bv_i$ are also useful if the call to the \lmo is more expensive. 
Note that for the nuclear norm special case $\cA_1 = \cA_2 = \{ \bx: \| \bx \|_2 = 1\}$, variants of such corrections of the atoms were studied by \citet{Laue:2012wn}.

In contrast to the non-convex corrections of atoms~$\bu_i$, the optimal updates of the weights $\balpha$ alone as in line~3 can be performed very efficiently after every atom update (as e.g. in OMP), by simply solving a linear system of size $r\times r$.

\section{Implementing the Linear Minimization Oracle for Matrix Factorizations}\vspace{-1mm}
\label{sec:LMO}
As in the Frank-Wolfe algorithm, the \lmo is required to return a minimizer of the function linearized at the current value of $\bX$. As we mentioned in the introduction, this type of \emph{greedy} oracle has been very widely used in several classical algorithms, including many variants of matching pursuit, Frank-Wolfe and related sparse greedy methods~\cite{Frank:1956vp,Jaggi:2010tz,tewari2011,wang2014matrixcompletion,Bach:2013wk}. 

Say $\bX_r\in \bbR^{n \times m}$ is the current iterate of algorithm or Algorithm~\ref{algo:GMP} on a matrix problem, for the objective function $f(\bX) = \| \bY - \bX \|_F^2$. In the symmetric case (the non-symmetric case is analogous), for arbitrary $\cA_1$, the \lmo can be equivalently written as finding the vector which solves\vspace{-1mm}
\begin{equation}
\label{eq:LMO}
\argmax_{\bu \in \cA_1} \ \langle - \nabla f(\bX_r), \bu \otimes \bu  \rangle
\end{equation} 
It is easy to see~\eqref{eq:LMO} represents a \emph{generalized} version of the top eigenvector problem (or singular vector, for the non-symmetric case). The problem in general is NP-hard for arbitrary atomic sets (such for example already in the simple case of unit-length non-negative vectors \cite{Murty:1987cy}). 
Specific atomic structures and assumptions can allow for an efficient \lmo. Nevertheless, we will here provide a very general technique to design the \lmo for arbitrary vector atoms, namely the \emph{atomic power method}. 
Our proposed method will iteratively run on the \lmo problem (which is non-convex in its variable vector $\bu$), in a an ascent fashion. As is the case with ascent methods on non-convex problems, the presented analysis will only show convergence to a fixed point. For hard \lmo problems, the atomic power method can be run several times with different initialization, and the best outcome can be chosen as the \lmo solution.
 
\subsection{The Atomic Power Method} 
\label{sec:AtomicPowerMethod}
We will first address the symmetric case of the non-convex \lmo problem~\eqref{eq:LMO}. 
We use the Frank-Wolfe algorithm~\cite{Frank:1956vp,jaggi13FW} with a fixed step size of 1 to approximate the \lmo problem. Although designed for constrained convex \emph{minimization}, it is known that using a fixed step size of 1 can make Frank-Wolfe methods suitable for constrained (non-convex) \emph{maximization} as in our formulation~\eqref{eq:LMO}~\cite{journee2010,luss2013}. 

To solve~\eqref{eq:LMO}, say $\bu^{(t)}$ is the $t^{\text{th}}$ iterate ($\bu^{(0)} \in \bbR^n$ is the initialization). Recall that, $ - \nabla f(\bX_r) = \bR_r$. The next Frank-Wolfe iterate is obtained as\vspace{-1mm}
\begin{equation}
\label{eq:innerLMOupdate}
\bu^{(t+1)} \leftarrow \argmax_{
\bu \in \cA_1} \,
\langle \bu, \bR_r \bu^{(t)} \rangle \vspace{-2mm}
\end{equation}

We call the update step~\eqref{eq:innerLMOupdate} an \emph{atomic} power iteration. It is easy to see that it recovers the standard power method as a special case, as well as the Truncated Power Method for sparse PCA suggested by~\citet{yuan2013}, the sparse power methods suggested by~\citet{luss2013}, and the cone constrained power method suggested by~\citet{deshpande2014}. It can be shown that the iterates monotonically increase the function value. 

\paragraph{Analysis.}
Our analysis is based on the techniques suggested by the work on convex constrained maximization by~\citet{journee2010} and~\citet{luss2013}. While their focus is on the sparse PCA setting, we apply their results to general sets of vector atoms.  Let $g(\bu) := \langle \bR_r, \bu \otimes \bu \rangle$ be the value of the \lmo problem for a given vector $\bu$, and $I(t):= \max_{\bu \in \cA_1} \langle \bY_r \bu^{(t)}, \bu - \bu^{(t)} \rangle$. Note that $I(t) \ge 0$ by definition. We assume $\forall \bu \in \cA_1$, $\langle \bR_r, \bu \otimes \bu \rangle \geq 0 $ so that $g(\cdot)$ is convex on $\conv(\cA_1)$. Or, a looser assumption to make is that all atoms in $\cA_1$ are normalized i.e. $ \forall \bu \in \cA_1, \bu^\top \bu = \text{const}$. Note that this assumption holds for most practical applications. If this is the case, $g(\cdot)$ can simply be made convex by defining it as $g(\bu):= \langle  \bR_r + \kappa \bI, \bu \otimes \bu \rangle$ for large enough $\kappa$. Adding a term that is constant for all atoms in $\cA_1$ does not change the maximizer.  

The following proposition is a consequence of Theorem~3.4 in the work of~\citet{luss2013}.
\begin{proposition}
\label{thm:monotonicconvergence}
If $\forall \bu \in \cA_1, \bu^\top \bR_r \bu \geq 0,$ and $\cA_1$ is compact, then, \vspace{-2mm}
\begin{enumerate}
\item[(a)] The sequence $\{g(\bu^{(t)})\}$ is monotonically increasing.
\item[(b)] The sequence $\{ I(t)\} \rightarrow 0$. 
\item[(c)] The sequence of iterates of atomic power method $\{\bu^{(t)}\}$ converges to a stationary point of $g(\cdot)$.
\end{enumerate}
\end{proposition}
\begin{proof}[Proof sketch]
(a) follows because of interval convexity of $g(\cdot)$ on $\cA_1$  which implies
\begin{align*}
g(\bu^{(t+1)}) & \geq   g(\bu^{(t)}) + \langle \bR_r \bu^{(t)},  \bu^{(t+1)} - \bu^{(t)} \rangle \\
&= g(\bu^{(t)}) + I(t),
\end{align*}
and because $I(t) \geq 0$ by definition. (b) follows because of (a) and because $g(\cdot)$ is upper bounded on $\cA_1$ (compactness assumption). (c) is a direct consequence from (b) by definition of $I(t)$. 
\end{proof}

Hence at each iteration, the value of the optimization problem increases unless $I(t) = 0$ which is the first order optimality condition till a fixed point is reached. 

For sharper analysis, we make further assumptions on $g(\cdot)$ and $\cA_1$ and provide corresponding convergence results. We assume a form of restricted strong convexity of the \lmo function on $\cA_1$. It is easy to see that this is equivalent to assuming $\big\langle \bR_r, \frac{\bu \otimes \bu}{\| \bu \|^2_2}  \big\rangle \geq \sigma_{\text{min}} > 0$, $\forall \, \bu \in \cA_1$.  This directly implies that $\nullSp(\bR_r) \cap \cA_1 = \{\}$. So, $\exists \, \gamma $ s.t. $\| \bR_r \bu \|_2 \geq \gamma >0 $.   Further, assume that $\conv(\cA_1)$ is $\delta$-strongly convex. Using the analysis developed in~\citep[Section 3.4]{journee2010}, we can give guarantees for convergence of the sequence $\{\bu^{(t)}\} $ developed in Proposition~\ref{thm:monotonicconvergence}. Say $g^\star$ a stationary point of the sequence $\{ g(\bu^{(t)})\}$. 

\begin{proposition} If, \vspace{-2mm}
\begin{itemize}
\item $g(\bu)$ is $\sigma_{\text{min}}-$strongly convex on $\cA_1$, and consequently has $\forall \bu \in \cA_1,  \| \bR_r \bu \|_2 \geq \gamma >0 $ for some $\gamma$, and \vspace{-1mm}
\item $\conv(\cA_1)$ is a $\delta$-strongly convex set with $\delta \geq 0$, \vspace{-2mm}
\end{itemize}
then, for $\gamma, \delta, \sigma_{\text{min}}$ as defined above, for $k \geq \frac{1}{\epsilon^2} \frac{g^\star - g(\bu^{(0)})}{ \gamma \delta + \sigma_{\text{min}} } $, the atomic power iterates converge as  $ \min_{k}\|  \bu^{(t+1)} - \bu^{(t)}  \|_2 \leq \epsilon$.\vspace{-2mm}
\end{proposition}
\begin{proof}
See~\citep[Theorem 4]{journee2010}.
\end{proof}

We discussed a generic way to design an \lmo for the symmetric case. For the non-symmetric case, a natural way to find the maximizing $\bu, \bv$ is by alternating maximization, which leads to an alternating power method. This is interesting given the success of alternating minimization methods for matrix factorization problems \cite{Hardt:2013uc}.
Alternatively, we can set $\bT =\smallmat{ 0 & \bR\\ \bR^\top & 0 }$, $\bt = \smallmat{ \bu \\ \bv}$ to reduce the non-symmetric case to the symmetric one, and our previous analysis can be re-applied.

\begin{figure*}[t]
\centering
\begin{subfigure}[b]{0.45\textwidth}
\includegraphics[scale=0.5]{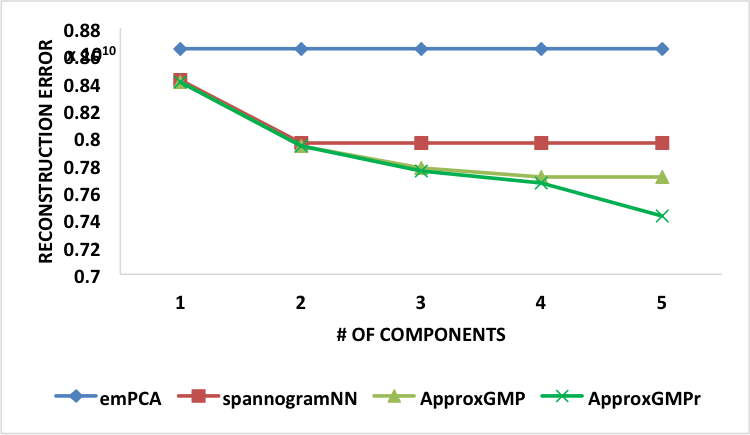}
\caption{Leukemia dataset}\vspace{-2mm}
\end{subfigure}
\centering
\begin{subfigure}[b]{0.45\textwidth}
\includegraphics[scale=0.5]{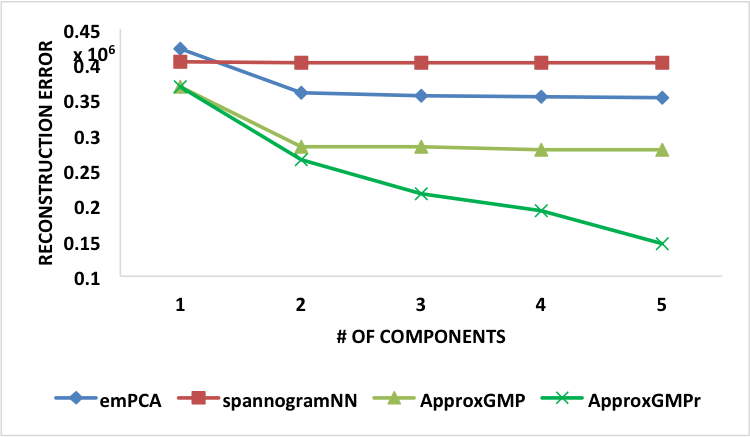}
\caption{CBCL Face train data}\vspace{-2mm}
\end{subfigure}
\caption{Reconstruction Error}\vspace{-4mm}
\label{fig:multiple}
\end{figure*}

\begin{figure*}
\centering
\begin{subfigure}[b]{0.45\textwidth}
\includegraphics[scale=0.5]{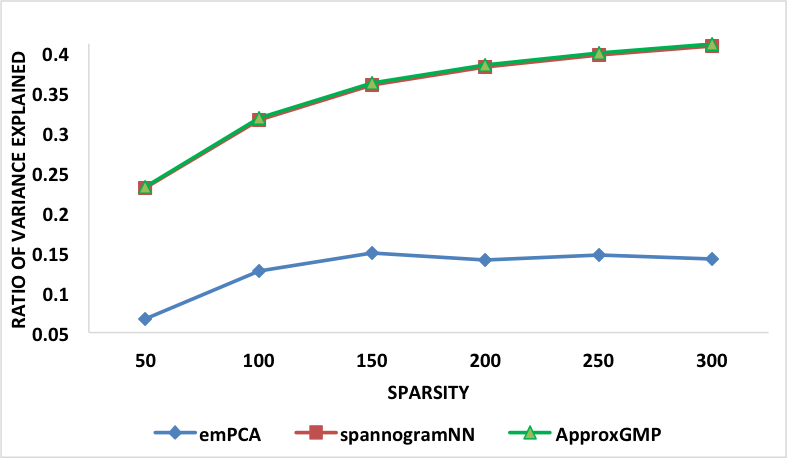}
\caption{Leukemia dataset}
\end{subfigure}
\begin{subfigure}[b]{0.45\textwidth}
\includegraphics[scale=0.5]{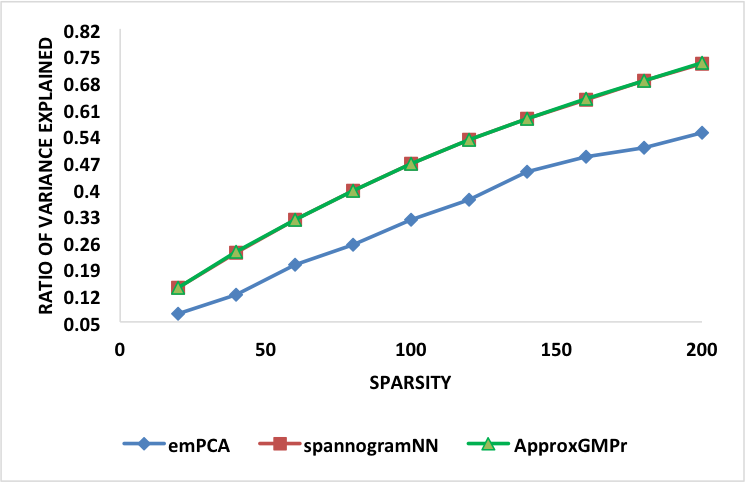}
\caption{CBCL Face train data}
\end{subfigure}
\caption{Ratio of variance explained vs. top-eigenvector}\vspace{3mm}
\label{fig:single}
\end{figure*}

\begin{table*}[h!]
\centering
\begin{tabular}{|c|c|c|c|}
\hline 
Dataset & ApproxGMP & ApproxGMPr & ApproxGMPr (Sparse) \\\hline 
Movielens100k & 1.778 $\pm$ 0.03 & 1.691 $\pm$ 0.03 & 1.62 $\pm$ 0.01 \\
Movielens1M & 1.6863  $\pm$ 0.01 & 1.6537 $\pm$ 0.01 & 1.6411 $\pm$ 0.01 \\
Movielens10M  & 1.8634 $\pm$ 0.01 & 1.8484  $\pm$ 0.01 &1.8452 $\pm$ 0.01 \\ \hline
\end{tabular}
\caption{RMSE on test set : average over 20 runs $\pm$ variance. For ApproxGMPr (Sparse), left singular vector is fully dense while the right one has sparsity 0.6 of its size (chosen by trial and error) }
\label{tab:matrixcompletion}\vspace{3mm}
\end{table*}

\section{Experiments}

In this section, we demonstrate some results on real world datasets. We provide empirical evidence for improvements for two algorithms: Truncated Power Method for sparse PCA and  matrix completion by rank-1 pursuit, by applying atom corrections suggested in Section~\ref{sec:corrective}. We also apply to the framework to sparse non-negative vectors as the atomic set to obtain a new algorithm using our proposed LMO design and compare it to existing methods. For the matrix completion, our framework yields a new general algorithm for structured matrix completion. We use two versions of our algorithm - ApproxGMP solves the LMO by power iterations without corrections, while  ApproxGMPr also uses corrections (Section~\ref{sec:corrective}). Note that GMP (solving the LMO exactly) is NP hard for the cases considered here (except for the setup of matrix completion by~\citet{wang2014matrixcompletion} which uses SVD), and is seldom used in practice, hence it is not compared against.

\paragraph{Sparse PCA: Inadequacy of Deflation.}
\label{sec:exptTpowerCorrections}
 Sparse PCA is a special symmetric case of~\eqref{eq:costfuncfactorization}, with the respective vector atom set defined as $\cA_1 = \{ \bu: \| \bu\|_2=1, \| \bu \|_0 \leq k\},$ where $k$ is the desired sparsity level defined externally as a parameter. 

 Since finding the top sparse eigenvector is an NP-hard problem in general, various practical algorithms only solve the problem approximately. As such, any deflation technique (see~\citep{Mackey_deflationmethods} for an overview of deflation for sparse PCA) coupled with finding an approximate top-eigenvector of the deflated matrix are still \emph{greedy} approaches similar to the un-corrected variant of Algorithm~\ref{algo:GMP}. This suggests that the optional atom corrections could be  useful.

To illustrate the utility of performing atom corrections, we consider the Truncated Power Method described by~\citet{yuan2013} which can be derived as a special case of the approximate LMO in Algorithm~\ref{algo:GMP}. We consider the Leukemia and CBCL face training datasets~\citep{lichman2013}. The Leukemia dataset has 72 samples, each consisting of expression values for 12582 probe sets. CBCL face training dataset contains 2429 images each represented as a feature vector of size 361 pixels.  For $r=5$ components each with sparsity 300, GMP obtains 0.3472 as the ratio of variance explained as opposed to 0.3438 by TPower+orthogonal deflation. Similarly for CBCL data, GMP and TPower obtain 0.7326 and 0.7231, respectively, for sparsity=200.\vspace{-2mm}

\paragraph{Sparse Non-Negative PCA.}
\label{sec:exptSNPCA}
For Sparse Non-Negative PCA, we use the Leukemia and CBCL Face training datasets as well.  The problem is a special symmetric case of~\eqref{eq:costfuncfactorization}, with the respective vector atom set defined as
$\cA_1 = \{ \bu: \| \bu\|_2=1, \| \bu \|_0 \leq k,  \bu \geq 0 \},$ 
where $k$ is the desired sparsity level defined externally as a parameter.  In this case study, in addition to using corrections as in the previous case study, we use the atomic power method described in Section~\ref{sec:AtomicPowerMethod} to derive a new algorithm for the atomic set defined above which to our knowledge has not been seen or studied before.

There is little prior work on Sparse Non-negative PCA. We compare against the spannogramNN by \citet{asteris2014} and emPCA by~\citet{sigg2008}. Our algorithm of atomic power iterations is easier to implement, converges faster and gives better results compared to both of these. Figure~\ref{fig:multiple} shows reconstruction error with increasing rank. Figure~\ref{fig:single} shows the ratio of variance explained for a rank one approximation for sparse non-negative PCA. We note that ApproxGMP has performance comparable to that of spannogramNN, and both ApproxGMP and spannogramNN outperform emPCA.

\vspace{-1mm}

\paragraph{Structured Low-Rank Matrix Completion.}
For $\cA_1 = \{ \bu \in \bbR^n: \| \bu\|_2 =1 \}$, $\cA_2 = \{ \bv \in \bbR^d: \| \bv\|_2 =1 \}$, and $f(\sum_i \alpha_i \bu_i \otimes \bv_i) :=  \| \bT_{\Omega} - \sum_i \alpha_i [\bu_i \otimes \bv_i]_\Omega \|_F$, GMP (Algorithm~\ref{algo:GMP}) can be used for matrix completion.  This recovers the work of \citet{wang2014matrixcompletion} who study this special case as their algorithm OR1MP and provide linear convergence guarantees. We empirically show that by using corrections, we get significantly better results in terms of reconstruction error. Furthermore, our more general analysis shows that the linear convergence holds for any structured sets. We consider the specific case of sparsity. \citet{soni14} study the regularization impact of sparse low rank matrix completion for robustness. Our framework yields a new algorithm for simple rank-one pursuit with alternating atomic power method for sparse factors. 
We used 3 movielens datasets of varying sizes for our experiments. In each dataset, we randomly split the ratings into 50-20-30 training, validation and testing split (we generate 20 different splits). The validation dataset is used for selecting the rank and applying further corrections for better generalization.  Our results (averaged over 20 runs) are reported in Table~\ref{tab:matrixcompletion}. We find that our generalizations of the OR1MP results in better reconstruction error for all three datasets. See the work by \citet{wang2014matrixcompletion} for a comparison of the performance of OR1MP with other matrix completion methods, and the work by~\citet{soni14} on robustness analysis for sparse factor low rank matrix completion.

\section*{Conclusion and Future Work}
We presented a pursuit framework for structured low rank matrix factorization and completion. Studying the tradeoff between rank and approximation quality, we proved linear convergence of generalized pursuit in Hilbert Spaces, of which matrix pursuit is a special case. Another direct application would be tensor pursuit for low rank tensor factorization and completion. A general design for the LMO construction for structured sets for tensors is an interesting future direction to explore. Moreover, generalization of the convergence results beyond distance functions is an interesting extension of the present work with many applications. Further, note that both the generalized pursuit and the Frank-Wolfe algorithms solve the same LMO to settle on the next best atom to add. We borrowed the idea of correcting already chosen atoms from the FW framework. Hence, studying the connection between the two frameworks should yield more insights in the future.

\newpage
{\small
\bibliographystyle{plainnat}
\bibliography{bibliography}

\begin{thebibliography}{48}
\providecommand{\natexlab}[1]{#1}
\providecommand{\url}[1]{\texttt{#1}}
\expandafter\ifx\csname urlstyle\endcsname\relax
  \providecommand{\doi}[1]{doi: #1}\else
  \providecommand{\doi}{doi: \begingroup \urlstyle{rm}\Url}\fi

\bibitem[Asteris et~al.(2014)Asteris, Papailiopoulos, and Dimakis]{asteris2014}
Megasthenis Asteris, Dimitris Papailiopoulos, and Alexandros Dimakis.
\newblock Nonnegative sparse pca with provable guarantees.
\newblock In \emph{Proceedings of the 31st International Conference on Machine
  Learning (ICML-14)}, pages 1728--1736. JMLR Workshop and Conference
  Proceedings, 2014.

\bibitem[Bach(2013)]{Bach:2013wk}
Francis Bach.
\newblock {Convex relaxations of structured matrix factorizations}.
\newblock \emph{arXiv.org}, 2013.

\bibitem[Bach et~al.(2008)Bach, Mairal, and Ponce]{Bach:2008wn}
Francis Bach, Julien Mairal, and Jean Ponce.
\newblock {Convex Sparse Matrix Factorizations}.
\newblock Technical report, 2008.

\bibitem[Baraniuk et~al.(2010)Baraniuk, Cevher, Duarte, and
  Hegde]{baraniuk2010}
Richard~G. Baraniuk, Volkan Cevher, Marco~F. Duarte, and Chinmay Hegde.
\newblock Model-based compressive sensing.
\newblock \emph{IEEE Trans. Inf. Theor.}, 56\penalty0 (4):\penalty0 1982--2001,
  April 2010.

\bibitem[Blumensath and Davies(2008)]{Blumensath:2008fs}
T~Blumensath and M~E Davies.
\newblock {Gradient Pursuits}.
\newblock \emph{Signal Processing, IEEE Transactions on}, 56\penalty0
  (6):\penalty0 2370--2382, June 2008.

\bibitem[Candes and Recht(2009)]{Candes:2009kj}
Emmanuel~J Candes and Benjamin Recht.
\newblock {Exact Matrix Completion via Convex Optimization}.
\newblock \emph{Foundations of Computational Mathematics}, 9\penalty0
  (6):\penalty0 717--772, April 2009.

\bibitem[Candes and Tao(2010)]{Candes:2010jb}
Emmanuel~J Candes and Terence Tao.
\newblock {The Power of Convex Relaxation: Near-Optimal Matrix Completion}.
\newblock \emph{IEEE Transactions on Information Theory}, 56(5) 2053--2080, 2010.

\bibitem[Candes et~al.(2011)Candes, Li, Ma, and Wright]{Candes:2011bd}
Emmanuel~J Candes, Xiaodong Li, Yi~Ma, and John Wright.
\newblock {Robust principal component analysis?}
\newblock \emph{Journal of the ACM}, 58\penalty0 (3), May 2011.

\bibitem[Chen et~al.(1989)Chen, Billings, and Luo]{chen1989orthogonal}
Sheng Chen, Stephen~A Billings, and Wan Luo.
\newblock Orthogonal least squares methods and their application to non-linear
  system identification.
\newblock \emph{International Journal of control}, 50\penalty0 (5):\penalty0
  1873--1896, 1989.

\bibitem[Davis et~al.(1997)Davis, Mallat, and Avellaneda]{Davis:1997fh}
G~Davis, S~Mallat, and M~Avellaneda.
\newblock {Adaptive greedy approximations}.
\newblock \emph{Constructive Approximation}, 13\penalty0 (1):\penalty0 57--98,
  March 1997.

\bibitem[Deshpande et~al.(2014)Deshpande, Montanari, and
  Richard]{deshpande2014}
Yash Deshpande, Andrea Montanari, and Emile Richard.
\newblock Cone-constrained principal component analysis.
\newblock In \emph{NIPS - Advances in Neural Information Processing Systems
  27}, pages 2717--2725. 2014.

\bibitem[DeVore and Temlyakov(2014)]{DeVore:2014ti}
Ronald~A. DeVore and Vladimir~N Temlyakov.
\newblock {Convex optimization on Banach Spaces}.
\newblock \emph{arXiv.org}, January 2014.

\bibitem[DeVore and Temlyakov(1996)]{devore08}
Ronald~A. DeVore and Vladimir~N. Temlyakov.
\newblock Some remarks on greedy algorithms.
\newblock \emph{Adv. Comput. Math.}, 5\penalty0 (1):\penalty0 173--187, 1996.

\bibitem[Dud{\i}k et~al.(2012)Dud{\i}k, Harchaoui, and Malick]{Dudik:2012ts}
Miroslav Dud{\i}k, Zaid Harchaoui, and Jerome Malick.
\newblock {Lifted coordinate descent for learning with trace-norm
  regularization￼}.
\newblock In \emph{AISTATS}, March 2012.

\bibitem[Dup{\'e}(2015)]{dupe2015}
Fran{\c c}ois-Xavier Dup{\'e}.
\newblock {Greed is Fine: on Finding Sparse Zeros of Hilbert Operators}.
\newblock working paper or preprint, February 2015.

\bibitem[Frank and Wolfe(1956)]{Frank:1956vp}
Marguerite Frank and Philip Wolfe.
\newblock {An Algorithm for Quadratic Programming}.
\newblock \emph{Naval Research Logistics Quarterly}, 3:\penalty0 95--110, 1956.

\bibitem[Gillis and Glineur(2011)]{Gillis:2011ia}
Nicolas Gillis and Fran{\c c}ois Glineur.
\newblock {Low-Rank Matrix Approximation with Weights or Missing Data Is
  NP-Hard}.
\newblock \emph{SIAM Journal on Matrix Analysis and Applications}, 32\penalty0
  (4):\penalty0 1149, 2011.

\bibitem[Gribonval and Vandergheynst(2006)]{Gribonval:2006ch}
R{\'e}mi Gribonval and P~Vandergheynst.
\newblock {On the exponential convergence of matching pursuits in
  quasi-incoherent dictionaries}.
\newblock \emph{IEEE Transactions on Information Theory}, 52\penalty0
  (1):\penalty0 255--261, 2006.

\bibitem[Hardt(2013)]{Hardt:2013uc}
Moritz Hardt.
\newblock {Understanding Alternating Minimization for Matrix Completion}.
\newblock \emph{arXiv.org}, December 2013.

\bibitem[Hazan(2008)]{Hazan:2008kz}
Elad Hazan.
\newblock {Sparse Approximate Solutions to Semidefinite Programs}.
\newblock In \emph{LATIN 2008}, pages 306--316. Springer Berlin Heidelberg,
  2008.

\bibitem[Jaggi(2013)]{jaggi13FW}
Martin Jaggi.
\newblock Revisiting Frank-Wolfe: Projection-free sparse convex optimization.
\newblock In \emph{ICML 2013 - Proceedings
  of the 30th International Conference on Machine Learning},
  volume~28, pages 427--435, 2013.

\bibitem[Jaggi and Sulovsk{\'{y}}(2010)]{Jaggi:2010tz}
Martin Jaggi and Marek Sulovsk{\'{y}}.
\newblock {A Simple Algorithm for Nuclear Norm Regularized Problems}.
\newblock \emph{ICML 2010: Proceedings of the 27th international conference on
  Machine learning}, 2010.

\bibitem[Jain et~al.(2013)Jain, Netrapalli, and Sanghavi]{jain2013}
Prateek Jain, Praneeth Netrapalli, and Sujay Sanghavi.
\newblock Low-rank matrix completion using alternating minimization.
\newblock In \emph{Proceedings of the Forty-fifth Annual ACM Symposium on
  Theory of Computing}, STOC '13, pages 665--674, New York, NY, USA, 2013. ACM.

\bibitem[{Jones}(1987)]{jones87}
Lee~K. {Jones}.
\newblock {On a conjecture of Huber concerning the convergence of projection
  pursuit regression.}
\newblock \emph{{Ann. Stat.}}, 15:\penalty0 880--882, 1987.
\newblock ISSN 0090-5364; 2168-8966/e.

\bibitem[Journ{\'e}e et~al.(2010)Journ{\'e}e, Nesterov, Richt\'{a}rik, and
  Sepulchre]{journee2010}
Michel Journ{\'e}e, Yurii Nesterov, Peter Richt\'{a}rik, and Rodolphe
  Sepulchre.
\newblock Generalized power method for sparse principal component analysis.
\newblock \emph{J. Mach. Learn. Res.}, 11:\penalty0 517--553, March 2010.
\newblock ISSN 1532-4435.

\bibitem[Koren et~al.(2009)Koren, Bell, and Volinsky]{Koren:2009jg}
Yehuda Koren, Robert Bell, and Chris Volinsky.
\newblock {Matrix Factorization Techniques for Recommender Systems}.
\newblock \emph{Computer}, 42\penalty0 (8):\penalty0 30--37, August 2009.

\bibitem[Lacoste-Julien and Jaggi(2015)]{LacosteJulien:2015wj}
Simon Lacoste-Julien and Martin Jaggi.
\newblock {On the Global Linear Convergence of Frank-Wolfe Optimization
  Variants}.
\newblock In \emph{NIPS 2015 - Advances in Neural Information Processing
  Systems 28}, pages 496--504, 2015.

\bibitem[Laue(2012)]{Laue:2012wn}
S{\"o}ren Laue.
\newblock {A Hybrid Algorithm for Convex Semidefinite Optimization}.
\newblock In \emph{ICML}, 2012.

\bibitem[Levy and Goldberg(2014)]{Levy:2014vd}
Omer Levy and Yoav Goldberg.
\newblock {Neural Word Embedding as Implicit Matrix Factorization }.
\newblock In \emph{NIPS 2014 - Advances in Neural Information Processing
  Systems 27}, 2014.

\bibitem[Lichman(2013)]{lichman2013}
M.~Lichman.
\newblock {UCI} machine learning repository, 2013.

\bibitem[Luss and Teboulle(2013)]{luss2013}
Ronny Luss and Marc Teboulle.
\newblock Conditional gradient algorithms for rank-one matrix approximations
  with a sparsity constraint.
\newblock \emph{SIAM Review Vol. 55 No. 1}, abs/1107.1163, 2013.

\bibitem[Mackey(2009)]{Mackey_deflationmethods}
Lester Mackey.
\newblock Deflation methods for sparse PCA.
\newblock In \emph{NIPS}, 2009.

\bibitem[Mallat and Zhang(1993)]{Mallat:1993gu}
St{\'e}phane Mallat and Zhifeng Zhang.
\newblock {Matching pursuits with time-frequency dictionaries}.
\newblock \emph{IEEE Transactions on Signal Processing}, 41\penalty0
  (12):\penalty0 3397--3415, 1993.

\bibitem[Murty and Kabadi(1987)]{Murty:1987cy}
Katta~G Murty and Santosh~N Kabadi.
\newblock {Some NP-complete problems in quadratic and nonlinear programming}.
\newblock \emph{Mathematical Programming}, 39\penalty0 (2):\penalty0 117--129,
  June 1987.

\bibitem[Papailiopoulos et~al.(2013)Papailiopoulos, Dimakis, and
  Korokythakis]{papailiopoulos_sparsepca}
Dimitris~S. Papailiopoulos, Alexandros~G. Dimakis, and Stavros Korokythakis.
\newblock Sparse PCA through low-rank approximations.
\newblock \emph{ICML}, 2013.

\bibitem[Pong et~al.(2010)Pong, Tseng, Ji, and Ye]{Pong:2010cg}
Ting~Kei Pong, Paul Tseng, Shuiwang Ji, and Jieping Ye.
\newblock {Trace Norm Regularization: Reformulations, Algorithms, and
  Multi-Task Learning}.
\newblock \emph{SIAM Journal on Optimization}, 20\penalty0 (6):\penalty0
  3465--3489, 2010.

\bibitem[Recht(2009)]{recht2009}
Benjamin Recht.
\newblock A simpler approach to matrix completion.
\newblock \emph{CoRR}, abs/0910.0651, 2009.

\bibitem[Recht et~al.(2010)Recht, Fazel, and Parrilo]{Recht:2010tf}
Benjamin Recht, Maryam Fazel, and Pablo~A Parrilo.
\newblock {Guaranteed Minimum-Rank Solutions of Linear Matrix Equations via
  Nuclear Norm Minimization}.
\newblock \emph{SIAM Review}, 52\penalty0 (3):\penalty0 471--501, 2010.

\bibitem[Richard et~al.(2014)Richard, Obozinski, and Vert]{richard2014}
Emile Richard, Guillaume~R Obozinski, and Jean-Philippe Vert.
\newblock Tight convex relaxations for sparse matrix factorization.
\newblock In \emph{Advances in Neural Information Processing Systems 27}, pages
  3284--3292. Curran Associates, Inc., 2014.

\bibitem[Sigg and Buhmann(2008)]{sigg2008}
Christian~D. Sigg and Joachim~M. Buhmann.
\newblock Expectation-maximization for sparse and non-negative pca.
\newblock In \emph{ICML}, pages 960--967, New York, NY, USA, 2008. ACM.
\newblock ISBN 978-1-60558-205-4.

\bibitem[Soni et~al.(2014)Soni, Jain, Haupt, and Gonella]{soni14}
Akshay Soni, Swayambhoo Jain, Jarvis~D. Haupt, and Stefano Gonella.
\newblock Noisy matrix completion under sparse factor models.
\newblock \emph{CoRR}, abs/1411.0282, 2014.

\bibitem[Temlyakov(2014)]{Temlyakov:2014eb}
Vladimir Temlyakov.
\newblock {Greedy algorithms in convex optimization on Banach spaces}.
\newblock In \emph{48th Asilomar Conference on Signals, Systems and Computers},
  pages 1331--1335. IEEE, 2014.

\bibitem[Tewari et~al.(2011)Tewari, Ravikumar, and Dhillon]{tewari2011}
Ambuj Tewari, Pradeep~K. Ravikumar, and Inderjit~S. Dhillon.
\newblock Greedy algorithms for structurally constrained high dimensional
  problems.
\newblock In \emph{Advances in Neural Information Processing Systems 24}, pages
  882--890, 2011.

\bibitem[Toh and Yun(2009)]{Toh:2009we}
K~Toh and S~Yun.
\newblock {An accelerated proximal gradient algorithm for nuclear norm
  regularized least squares problems}.
\newblock \emph{Optimization Online}, 2009.

\bibitem[Tropp(2004)]{Tropp:2004gc}
Joel~A Tropp.
\newblock {Greed is good: algorithmic results for sparse approximation}.
\newblock \emph{IEEE Transactions on Information Theory}, 50\penalty0
  (10):\penalty0 2231--2242, 2004.

\bibitem[Wang et~al.(2014)Wang, jun Lai, Lu, Fan, Davulcu, and
  Ye]{wang2014matrixcompletion}
Zheng Wang, Ming jun Lai, Zhaosong Lu, Wei Fan, Hasan Davulcu, and Jieping Ye.
\newblock Rank-one matrix pursuit for matrix completion.
\newblock In \emph{ICML-14 - Proceedings of the 31st International Conference on Machine
  Learning}, pages 91--99,
  2014.

\bibitem[Yang et~al.(2015)Yang, Mehrkanoon, and Suykens]{Yang:2015wy}
Yuning Yang, Siamak Mehrkanoon, and Johan A~K Suykens.
\newblock {Higher order Matching Pursuit for Low Rank Tensor Learning}.
\newblock \emph{arXiv.org}, March 2015.

\bibitem[Yuan and Zhang(2013)]{yuan2013}
Xiao-Tong Yuan and Tong Zhang.
\newblock Truncated power method for sparse eigenvalue problems.
\newblock \emph{J. Mach. Learn. Res.}, 14\penalty0 (1):\penalty0 899--925,
  April 2013.
\newblock ISSN 1532-4435.

\end{thebibliography}
}

\appendix
\clearpage

\section{Proofs}

\subsection{Proof of Theorem \ref{thm:convExactHilbert}}

We prove that Algorithm~\ref{algo:GP} converges linearly for $\bx_0$ to~$\bx^\star$. To bound the convergence rate, we need to study how the residual changes over iterations. To this end, we define $\bq_i := \br_{i} - \br_{i+1}$. We start by stating auxiliary results that will be used to prove Theorem \ref{thm:convExactHilbert}. Recall from Algorithm~\ref{algo:GP} that $\bz_i = \lmo( - \br_i) = \argmax_{\bz \in \cS} \langle \bz, \br_i \rangle$.

\begin{proposition}
Let $\br_i \in \cH$ be a residual, i.e., $\br_i = \by - \sum_{j < i} \alpha_j \bz_j$,
then  $\max_{\bz \in \cS} \langle \br_i, \bz \rangle_\cH = 0 $ iff $\br_i = \br^\star$.\end{proposition} 
\begin{proof}
Follows from the first-order optimality condition for $\br^\star$. 
\end{proof}

\begin{proposition}
$\langle \br_i, \bz_j \rangle = 0,\,\, \forall j < i$.
\end{proposition}
\begin{proof}
Follows from the first-order optimality condition for $\balpha$. 
\end{proof}

\begin{proposition}
$\langle \br_i, \bq_j \rangle = 0, \,\, \forall j < i$.
\end{proposition}
\begin{proof}
By definition $\bq_j \in \lin(\{\bz_1, \ldots, \bz_j\})$. 
\end{proof}

\begin{proposition}
\label{thm:lowerBoundLHilbert}
\begin{equation}
\| \bq_i \|_\cH^2 \geq \frac{| \langle \bz_i, \br_i \rangle_\cH |^2}{\| \bz_i \|_\cH^2} \nonumber
\end{equation}
\end{proposition}
\begin{proof}
Let $\sfP_i$ be the orthogonal projection operator to $\lin(\{\bz_1, \dots , \bz_i\})$, i.e., for $\bx \in \cH$ we have $\sfP_i \bx = \arg \min_{\bz \in \lin(\{\bz_1, \dots , \bz_i\})} \| \bz - \bx \|_\cH^2$. Hence, $\bx_i =  \sfP_i \by$ and $\br_{i} = (\sfI - \sfP_{i-1}) \by$, where $\sfI$ designates the identity operator on~$\cH$. By the Gram-Schmidt process we get
\begin{eqnarray*}
\br_{i+1} &=& \underbrace{(\sfI - \sfP_{i-1}) \by}_{= \br_{i}}  \\
&+&\left\langle \frac{(\sfI - \sfP_{i-1}) \bz_i}{\| (\sfI - \sfP_{i-1}) \bz_i \|_\cH} , \by \right\rangle_\cH \frac{(\sfI - \sfP_{i-1}) \bz_i}{\| (\sfI - \sfP_{i-1}) \bz_i \|_\cH} \nonumber
\end{eqnarray*}
which in turn implies
\begin{align*}
\| \bq_i \|_\cH^2 &= \left\| \left\langle \frac{(\sfI - \sfP_{i-1}) \bz_i}{\| (\sfI - \sfP_{i-1}) \bz_i \|_\cH} , \by \right\rangle_\cH \frac{(\sfI - \sfP_{i-1}) \bz_i}{\| (\sfI - \sfP_{i-1}) \bz_i \|_\cH} \right\|_\cH^2 \nonumber \\
& = \left| \left\langle \frac{(\sfI - \sfP_{i-1}) \bz_i}{\| (\sfI - \sfP_{i-1}) \bz_i \|_\cH} , \by \right\rangle_\cH \right|^2 \nonumber \\
&=  \frac{ | \langle \bz_i , (\sfI - \sfP_{i-1}) \by \rangle_\cH |^2}{\| (\sfI - \sfP_{i-1}) \bz_i \|_\cH^2} \\\nonumber
&= \frac{ | \langle \bz_i ,\br_i \rangle_\cH |^2}{\| (\sfI - \sfP_{i-1}) \bz_i \|_\cH^2} \\ \nonumber
&\geq \frac{ | \langle \bz_i ,\br_i \rangle_\cH |^2}{\| \bz_i \|_\cH^2}, \nonumber
\end{align*}
where we used that $(\sfI - \sfP_{i-1})$ is a self-adjoint operator to obtain the third equality, and $\|(\sfI - \sfP_{i-1})\|_\mathrm{op} = 1$ to get the inequality.
\end{proof}

We are now ready to prove Theorem \ref{algo:GP}.

\begin{theorem}[Theorem \ref{algo:GP}]
\label{thm:convExactHilbertagain}
Algorithm~\ref{algo:GP} converges linearly for $f_\cH(\bx) := \frac{1}{2}d^2_\cH (\bx, \by)$, $\by \in \cH$.
\end{theorem}

\begin{proof}
\begin{align}
\| \br_{i+1} - \br^\star \|^2_\cH
&= \| \br_{i}  - \br^\star  - \bq_i \|^2_\cH \nonumber \\
&= \| \br_{i} - \br^\star \|^2_\cH + \| \bq_i\|^2_\cH - 2 \langle \br_{i}  - \br^\star, \bq_i \rangle_\cH \nonumber \\
&= \| \br_{i} - \br^\star \|^2_\cH + \| \bq_i\|^2_\cH \nonumber\\
& \qquad - 2 \langle \br_{i+1} + \bq_i - \br^\star, \bq_i \rangle_\cH \nonumber\\
&= \| \br_{i} - \br^\star \|^2_\cH + \| \bq_i\|^2_\cH  - 2 \| \bq_i\|^2_\cH \nonumber\\ 
&= \| \br_{i} - \br^\star \|^2_\cH -  \| \bq_i\|^2_\cH  \nonumber\\
& \leq   \| \br_{i} - \br^\star \|^2_\cH  -  \frac{\langle \bz_i, \br_i \rangle_\cH^2}{\|\bz_i\|_\cH^2}\,\, (\text{Prop}~\ref{thm:lowerBoundLHilbert}) \nonumber\\
&=   \| \br_{i} - \br^\star \|^2_\cH  -  \frac{\langle \bz_i, \br_i - \br^\star\rangle_\cH^2}{\|\bz_i\|_\cH^2} \label{eq:pfthmcoherece}\\
&= \left( 1 - \frac{\langle \bz_i, \br_i - \br^\star\rangle_\cH^2}{\|\bz_i\|_\cH^2 \| \br_{i} - \br^\star \|^2_\cH} \right)  \| \br_{i} - \br^\star \|^2_\cH \nonumber\\
& = \mu  \| \br_{i} - \br^\star \|^2_\cH
\end{align}
$\mu \in [0,1) $ by Cauchy-Schwarz.
\end{proof}

We finally note that proving convergence of Algorithm \ref{algo:GP} \emph{without} specifying a rate can be achieved as follows.

\begin{proposition}
The sequence of residuals $\{\br_i\}$ produced by Algorithm \ref{algo:GP} with $f_\cH(\bx) := \frac{1}{2}d^2_\cH (\bx, \by)$, $\by \in \cH$, converges to $\br^\star$.
\end{proposition} 
\begin{proof}
We have 
\begin{align*}
\| \br_{i+1}\|_\cH^2 &= \| \br_i - \bq_i \|_\cH^2 \\
&= \| \br_i \|_\cH^2  + \| \bq_i \|_\cH^2 - 2 \langle \br_i, \bq_i \rangle_\cH \\
&=  \| \br_i \|_\cH^2 +\| \bq_i \|_\cH^2 - 2 \langle \br_{i+1} + \bq_i, \bq_i \rangle_\cH \\
&=  \| \br_i \|_\cH^2 +  \| \bq_i \|_\cH^2 - 2 \langle \bq_i, \bq_i \rangle_\cH \\
&=  \| \br_i \|_\cH^2 -  \| \bq_i \|_\cH^2 
\end{align*}
From Prop~\ref{thm:lowerBoundLHilbert},
$\| \br_{i+1}\|_\cH^2 \leq \| \br_i \|_\cH^2 -   \frac{\langle \bz_i, \br_i\rangle_\cH^2 }{ \| \bz_i \|_\cH^2} $. Since $\langle \bz_i, \br_i \rangle_\cH^2 = 0 $ iff $\br_i = \br^\star$  the sequence $\{\| \br_i\|_\cH\}$ monotonically decreases with increasing $i$ until it converges to the lower bound $\|\br^\star\|_\cH$. 
\end{proof}

\subsection{Proof of Theorem~\ref{thm:coherence}}

Our proofs rely on the Gram matrix $\bG(\cJ)$ of the atoms in~$\cS$ indexed by $\cJ \subseteq [n]$, i.e., $(\bG(\cJ))_{i,j} := \langle \bs_i, \bs_j \rangle_\cH, i,j \in \cJ$.

To prove Theorem~\ref{thm:coherence}, we use the following known results. 

\begin{lemma}[\citet{Tropp:2004gc}] The smallest eigenvalue $\lambda_{\min} (\bG(\cJ))$ of $\bG(\cJ)$ obeys $\lambda_{\min}(\bG(\cJ)) > 1 - \mu(m-1)$, where $m = |\cJ|$.
\label{thm:eigenval}
\end{lemma}

\begin{lemma}[\citet{devore08}] For every index set $\cJ \subseteq [n]$ and every linear combination $\bp$ of the atoms in $\cS$ indexed by $\cJ$, i.e., $\bp := \sum_{j \in \cJ} v_j \bs_j$, we have $\max_{j \in \cJ } | \langle \bp, \bs_j  \rangle_\cH | \ge \frac{\| \bp \|_\cH^2}{\|\bv\|_1} = \frac{\langle \bv, \bG(\cJ) \bv \rangle_2 }{\| \bv\|_1}$, where $\bv \neq \bf 0$ is the vector having the $v_j$ as entries.
\label{thm:lowerbounddotprod}
\end{lemma}
 
\begin{proof}[Proof of Theorem~\ref{thm:coherence}]
Note that $\bm_i := \br_i - \br^\star =  (\by - \bx^{(i-1)}) - (\by - \bx^\star)  = \bx^\star - \bx^{(i-1)} \in \text{lin}(\cS), \, \forall \, i$, by assumption, which implies that there exists a vector $\bv_i \neq \bf 0$ s.t. $\bm_i =  \sum_{j \in [n]} (\bv_i)_j \bs_j$. 
Setting $\cJ = [n]$, we have

\begin{eqnarray*}
| \langle \bz_i, \br_i - \br^\star \rangle_\cH | 
 &\overset{\text{by def. of $\bz_i$}}{
=} &\bigg| \max_{j \in [n]} (\langle \bs_j, \br_i\rangle_\cH \! - \! \underbrace{\langle \bs_j, \br^\star \rangle_\cH}_{= 0}) \bigg| \\
& \overset{\text{\hspace{-0.3cm}symmetry of $\cS$\hspace{-0.3cm}}}{
=} &\max_{j \in [n]} | \langle \bs_j, \br_i\rangle_\cH \! - \! \langle \bs_j, \br^\star \rangle_\cH | \\
&=&\max_{j \in [n]} | \langle \bs_j,\bm_i \rangle_\cH | \\
&\overset{\text{Lemma~\ref{thm:lowerbounddotprod}}}{
\geq}& \frac{\| \bm_i \|^2_\cH}{\| \bv_i \|_1}\\
&\geq& \frac{\| \bm_i \|^2_\cH}{ \sqrt{n} \| \bv_i \|_2} \\
&=&\frac{\left(\sqrt{\langle \bv_i, \bG(I) \bv_i \rangle_2}\right) \| \bm_i \|_\cH}{ \sqrt{n} \| \bv_i \|_2} \\
&\overset{\text{Lemma~\ref{thm:eigenval}}}{
\geq}& \sqrt{\frac{1 - \mu(n-1)}{n}} \|\br_i - \br^\star\|_\cH.
\end{eqnarray*}

Replacing the second term in \eqref{eq:pfthmcoherece} in the proof of Theorem~\ref{thm:convExactHilbertagain} with the last expression above, and noting that $\| \bz_i \|_\cH = 1$, we get the desired result.

To get the result in terms of $\mu(1)$, note that $1 - \frac{1 - \mu(n-1)}{n} \leq (1 -1/n)(1 + \mu(1))$.
\end{proof}

\subsection{Proof of Corollary~\ref{cor:coherenceomp}}
\begin{proof}
The proof is similar to that of Theorem~\ref{thm:coherence} with $\br^\star = \bf0$ (which implies that $\br_i $ lies in $\lin(\cS)$).
\end{proof}

\subsection{Inexact \lmo}\label{sec:inexactLMO}

Instead of solving the \lmo in each iteration exactly (which may be prohibitively expensive), it is often more realistic to obtain a $\delta_i \in (0,1]$ approximate solution to the \lmo at iteration~$i$. In other words, in each iteration our update is $\tilde{\bz}_i$ instead of $\bz_i$ so that the following holds (for simplicity all other notations including those of the residual vectors are overloaded) 
\begin{equation}
\label{eq:inexactLMO}
\langle \tilde{\bz}_i , \br_i \rangle_\cH \geq \delta_i  \langle {\bz}_i , \br_i \rangle_\cH.
\end{equation}

Note that the $\balpha $ update in each iteration is still exact. To make the effect of the inexact \lmo on the rate explicit, we assume that the atoms are normalized according to $\| \bs \|_\cH = c, \forall \bs \in \cS,$ for some constant $c > 0$. We emphasize that linear convergence guarantees can be obtained for the inexact \lmo even without this assumption.  
Proceeding as in the proof of Proposition~\ref{thm:lowerBoundLHilbert}, we get a slightly weaker lower bound for $\| \bq_i \|_\cH^2$, namely
\begin{equation*}
\| \bq_i \|_\cH^2 \geq \frac{\langle \tilde{\bz}_i , \br_i \rangle_\cH^2}{\| \tilde{\bz}_i \|^2_\cH} \geq \delta^2_i \frac{\langle  \bz_i , \br_i \rangle_\cH^2}{\|\bz_i\|_\cH^2},
\end{equation*}
where we used $\| \tilde{\bz}_i \|_\cH = \|\bz_i\|_\cH$.
We now obtain the following.
\begin{theorem}[Linear convergence with inexact \lmo] If the \lmo in Algorithm~\ref{algo:GP} is solved within accuracy $\delta_i$ as in~\eqref{eq:inexactLMO}, then Algorithm~\ref{algo:GP} converges with a linear rate.
\label{thm:convergenceInexact}
\end{theorem}
\begin{proof}
We proceed as in the proof of Theorem~\ref{thm:convExactHilbert} to get
\begin{align*}
\| \br_{i+1} - \br^\star \|_\cH^2 \!
 &\leq  \| \br_i - \br^\star\| _\cH^2 - \delta_i^2\langle {\bz}_i, \br_i \rangle_\cH^2\\
 &=   \| \br_{i} - \br^\star \|^2_\cH  -  \delta_i^2 \frac{\langle \bz_i, \br_i - \br^\star\rangle_\cH^2}{\|\bz_i\|_\cH^2} \\
 &= \!  \left(1  -  \delta_i^2\frac{\langle \bz_i, \br_i - \br^\star\rangle_\cH^2}{\|\bz_i\|_\cH^2 \| \br_{i} - \br^\star \|^2_\cH}\right) \| \br_{i} - \br^\star \|^2_\cH,
\end{align*}
 from which the result follows.
\end{proof}

\end{document}